\documentclass[11pt]{article}

\usepackage[margin=1in]{geometry}
\usepackage{amsmath,amssymb,amsthm,mathtools,bm}
\usepackage{graphicx}
\usepackage{booktabs}
\usepackage{microtype}
\usepackage{caption}
\usepackage{algorithm}
\usepackage{algpseudocode}
\usepackage[authoryear]{natbib}
\usepackage[hidelinks]{hyperref}
\usepackage[capitalise,noabbrev]{cleveref}

\graphicspath{{figures/}}
\newcommand{\E}{\mathbb{E}}
\newcommand{\Prb}{\mathbb{P}}
\newcommand{\Nset}{\mathcal{N}}
\newcommand{\Aset}{\mathcal{A}}
\newcommand{\Sset}{\mathcal{S}}
\newcommand{\Hset}{\mathcal{H}}

\newcommand{\clip}{\operatorname{clip}}
\newcommand{\argmax}{\operatorname*{arg\,max}}

\newcommand{\VPI}{\operatorname{VPI}}

\theoremstyle{plain}
\newtheorem{theorem}{Theorem}
\newtheorem{lemma}{Lemma}
\newtheorem{proposition}{Proposition}
\newtheorem{corollary}{Corollary}
\theoremstyle{definition}
\newtheorem{assumption}{Assumption}
\newtheorem{definition}{Definition}
\theoremstyle{remark}
\newtheorem{remark}{Remark}

\title{\textbf{PROSE: A Theory of Optimal Stopping with\\[2pt]
Perishable Evidence for Peer Selection in\\
Intermittently Connected Decentralised Learning}}

\author{
Christos Anagnostopoulos\thanks{School of Computing Science, University of Glasgow, Glasgow, United Kingdom. Email: \texttt{christos.anagnostopoulos@glasgow.ac.uk}.}
}
\date{}

\begin{document}
\maketitle

\begin{abstract}
\noindent
Decentralised federated learning removes the aggregation server but makes
collaboration dependent on transient peer availability. In mobile and
intermittently connected systems, evaluating a promising peer consumes contact
time and may cause the exchange opportunity itself to vanish, so that the
evidence a learner gathers about a peer is \emph{perishable}: it decays because
links expire and because peer models drift while old measurements age. This
paper develops a self-contained theory of optimal stopping for the resulting
peer-selection problem. We formalise a receiver's within-contact decision as a
finite-horizon Markov optimal-stopping problem with costly information
acquisition and a future-arrival outside option, and prove that it admits an
optimal policy characterised by a reservation value (Snell-envelope structure).
Around this formulation we prove: (i) stage-uniform, drift-aware concentration
and a maximin \emph{certification} rule that is correct with high probability
together with a finite-sample identification bound; (ii) a mobility-aware
value-of-information stopping rule and comparative statics showing that higher
link hazard lowers the value of continued probing and enlarges the stopping
region; (iii) a closed-form value of waiting under marked-Poisson contact
arrivals, together with a search-theoretic reservation value whose comparative
statics we characterise; and (iv) a myopic-optimality theorem establishing that,
in sufficiently volatile (monotone) mobility regimes, the one-step
confidence-safe rule is a sound surrogate for the optimal policy and never stops
prematurely. We instantiate the theory as \textbf{PROSE} (\emph{Perishable-evidence
Reservation-value Optimal Stopping for Exchange}), a lightweight, fully local
policy, and delineate the static-contact and drift-free limits in which classical
sequential decision problems are recovered. The development is entirely analytical:
the contribution is a theory of when, whether, and with whom a model exchange
should occur under perishable evidence.

\end{abstract}

% ==================================================================== section 1
\section{Introduction}
\label{sec:intro}

Federated learning (FL) trains models from distributed data without centralising
raw examples \citep{mcmahan2017,kairouz2021}. Classical FL assumes an
orchestration server, whereas decentralised variants replace server-mediated
aggregation with peer-to-peer exchanges or consensus over a communication graph
\citep{lian2017,hegedus2019,savazzi2020,li2022defkt,hashemi2022}. Removing the
aggregation point moves an important decision to the devices themselves:
\emph{which reachable peer should be used for the next model exchange?}

The question is sharper when the graph is mobile. In vehicular, wearable,
robotic, or opportunistic networks, the neighbour set of device $i$ is not fixed
but evolves as $\Nset_i(t)$, and recent work studies decentralised FL over
time-varying mobile-computing graphs \citep{li2026tmc}, mobility-aware vehicular
FL \citep{chen2025mobility}, and participation under high mobility
\citep{tu2026dsfl}. These methods establish that topology volatility directly
affects who can train together, for how long, and at what communication cost.
Most topology-aware methods, however, treat the quality or utility of an
available collaboration as known, immediately measurable, or embedded in an
optimisation or reinforcement-learning state.

In practice, peer quality is itself uncertain and must be \emph{acquired}. A
receiver may hold a small non-private anchor set and request logits or a scalar
compatibility score from a candidate peer \emph{before} sending a full model.
Such a probe is far cheaper than a model-bearing exchange, but it is not free: it
consumes time, bandwidth, and, crucially, \emph{contact lifetime}. Under
transient connectivity this creates a decision problem absent from static peer
selection. A promising peer can vanish during an additional probe; a highly
compatible peer can be a poor practical choice if the link is unlikely to survive
a model transfer; and delaying an exchange can be rational when new contacts are
likely to arrive soon. The value of one more observation therefore depends
jointly on \emph{what could be learned} and on \emph{whether the opportunity
survives long enough to exploit it}. We call evidence with this dual fragility
\emph{perishable}: it decays because links expire (an opportunity clock) and
because a peer keeps training while disconnected, so that old measurements
describe a stale model (a drift clock). \Cref{fig:setting,fig:timeline}
illustrate the setting and one decision episode.

\paragraph{Contribution and scope.}
This paper is a \emph{theory} paper. Rather than proposing a system, it isolates
and analyses the sequential decision that occurs \emph{inside a contact
opportunity} and treats peer evidence as perishable. Our contributions are as
follows.
\begin{enumerate}
\item \textbf{An optimal-stopping formulation.} We model the receiver's decision
as a finite-horizon Markov optimal-stopping problem with four actions
---\emph{exchange}, \emph{probe}, \emph{wait}, and \emph{random}---costly
information acquisition, and a future-arrival outside option. We prove existence
of an optimal policy and a dynamic-programming (Snell-envelope) characterisation,
and show that the optimal policy is a reservation-value threshold rule
(\Cref{sec:formulation}).
\item \textbf{Confidence-safe certification.} We derive stage-uniform,
drift-aware confidence bounds and a maximin certification rule that selects a
peer only when its \emph{realisable} value is provably better than every
alternative. We prove correctness with high probability and a finite-sample
identification bound (\Cref{sec:certify}).
\item \textbf{Mobility-aware value of information.} We prove a sufficient
do-not-probe stopping rule that charges a probe both its direct cost and the
opportunity loss caused by possible disconnection, and establish comparative
statics: higher link hazard lowers continuation value, enlarges the stopping
region relative to the waiting and fallback options, and monotonically lowers the
acquisition index (\Cref{sec:voi}).
\item \textbf{The waiting option.} We give a closed-form value of waiting under
marked-Poisson contact arrivals and, in the stationary limit, a search-theoretic
reservation value whose monotonicity in arrival rate, delay cost, and mark
distribution we characterise (\Cref{sec:wait}).
\item \textbf{Myopic optimality.} We prove that in sufficiently volatile
(\emph{monotone}) mobility regimes the exact one-step-look-ahead rule is optimal,
and that the implementable confidence-safe rule is a \emph{sound} surrogate:
every state in which it halts probing is one in which halting is optimal, so it
never stops prematurely (\Cref{sec:myopic}).
\end{enumerate}
We instantiate the theory as a lightweight, fully local policy, PROSE
(\Cref{sec:policy}), and describe the classical limits it recovers
(\Cref{sec:limits}). The treatment is entirely analytical throughout.

\begin{figure}[t]
  \centering
  \includegraphics[width=\linewidth]{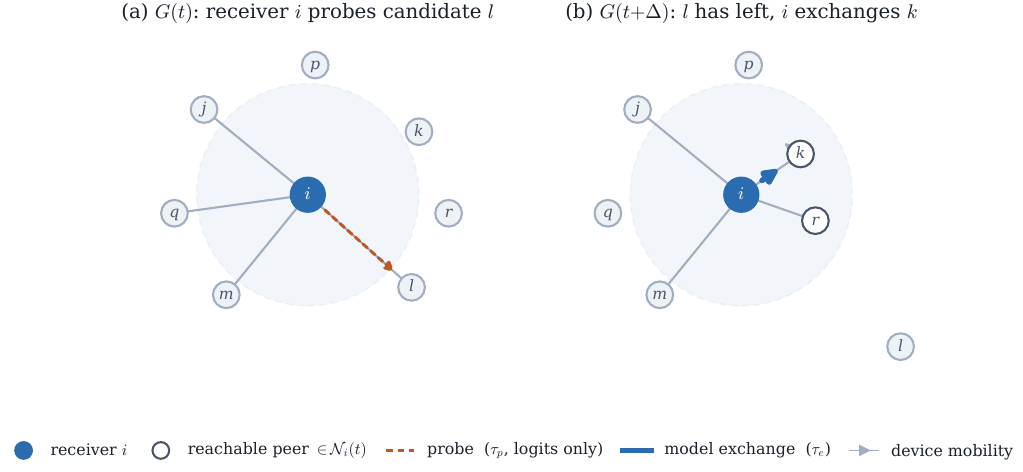}
  \caption{Decentralised learning over a temporal contact graph
  $G(t)=(\mathcal V,\mathcal E(t))$. Receiver $i$ (blue) can probe or exchange
  only with peers currently inside its communication range,
  $\Nset_i(t)$. Between $t$ and $t+\Delta$ mobility rewrites the neighbour set:
  candidate $l$ leaves before an exchange can occur, while $k$ becomes reachable.
  A probe transfers only anchor logits ($\tau_p$); an exchange transfers the full
  model ($\tau_e\gg\tau_p$). The decision is which action to take, and when.}
  \label{fig:setting}
\end{figure}

\begin{figure}[t]
  \centering
  \includegraphics[width=\linewidth]{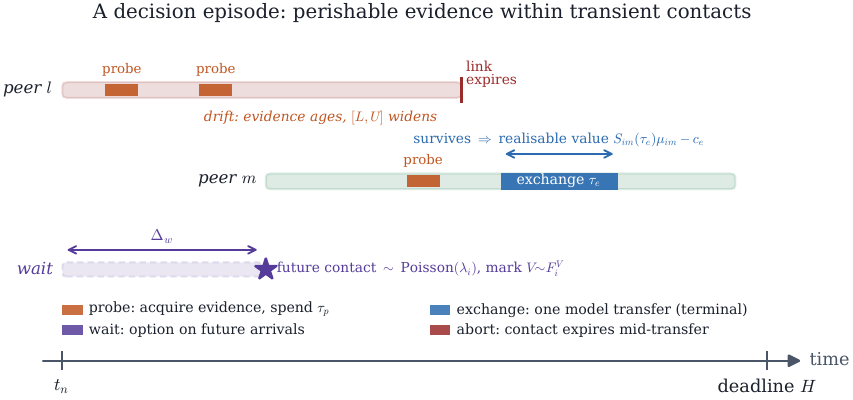}
  \caption{A single decision episode at receiver $i$ over the horizon
  $[t_n,H]$. Peer $l$ is probed twice; its evidence ages (drift widens its
  interval) and its link expires before any transfer, an \emph{aborted}
  opportunity. Peer $m$ is probed once and then exchanged because the contact is
  estimated to survive the transfer duration $\tau_e$. Waiting holds an option on
  a future contact arriving according to a marked-Poisson process. Only one
  model-bearing transfer is initiated per episode.}
  \label{fig:timeline}
\end{figure}

% ==================================================================== section 2
\section{Related Work}
\label{sec:related}

\paragraph{Decentralised and gossip-based learning.}
Decentralised optimisation replaces the parameter server with local communication
over a graph. D-PSGD established convergence and systems advantages for
decentralised stochastic optimisation \citep{lian2017}; gossip learning exchanges
models between peers without a central aggregator and is designed for settings
with churn \citep{hegedus2019}. Consensus-based serverless FL
\citep{savazzi2020}, knowledge-transfer-based decentralised FL
\citep{li2022defkt}, and communication-constrained decentralised FL with multiple
gossip steps \citep{hashemi2022} show that the communication graph and exchange
protocol shape learning behaviour. These works motivate peer-to-peer learning but
do not analyse how much evidence should be acquired about a transient peer before
committing a full model exchange.

\paragraph{Time-varying and mobile FL.}
Dynamic connectivity is now a first-class FL concern. DACFL studies dynamic
average consensus over time-varying sensor networks \citep{dacfl2022};
\citet{li2026tmc} jointly address time-varying topology and heterogeneity through
topology learning and resource optimisation; \citet{chen2025mobility} formulate
mobility-aware vehicular FL as a Dec-POMDP solved by multi-agent PPO; and
\citet{tu2026dsfl} use dynamic sketches for high-mobility participation. Our focus
is complementary: rather than learning a topology, leader policy, or sketch
repository, we analyse the sequential decision inside a contact opportunity when
the candidate itself may disappear before an observation can be exploited.

\paragraph{Optimal stopping, sequential analysis, and search.}
Optimal stopping formalises when the expected benefit of further observation no
longer justifies delay or acquisition cost \citep{peskir2006,ferguson2006}. The
theory has deep roots in sequential analysis \citep{wald1947} and statistical
decision theory \citep{degroot1970}, and the monotone case admits a
one-step-look-ahead solution \citep{chow1971}. The waiting option is a search
problem in the spirit of the McCall job-search model \citep{mccall1970}, in which
a reservation value separates acceptance from continued search. Our confidence
analysis uses concentration for bounded observations \citep{hoeffding1963}, and
the certification objective is a fixed-confidence best-arm identification problem
\citep{evendar2006,jamieson2014}. The novelty here is to make continuation value
depend explicitly on residual contact survival and on the option value of future
peer arrivals, producing a stopping criterion in which information is perishable
both because links expire and because model quality drifts while old probes age.

% ==================================================================== section 3
\section{The Perishable-Evidence Selection Problem}
\label{sec:model}

\subsection{Decentralised learning on a temporal graph}
Consider $N$ devices $\mathcal V=\{1,\ldots,N\}$. Device $i$ holds private data
$D_i$ and local objective $F_i(w)=|D_i|^{-1}\sum_{z\in D_i}\ell(w;z)$, and the
population objective is $F(w)=\sum_i p_iF_i(w)$; no server aggregates all updates.
Each device maintains a local model $w_i$. Communication is described by a
temporal graph $G(t)=(\mathcal V,\mathcal E(t))$ with
$\Nset_i(t)=\{j:(i,j)\in\mathcal E(t)\}$. A device alternates local optimisation
with opportunistic exchange: if $i$ receives peer model $w_j$, a generic pairwise
update is $w_i^{+}=(1-\alpha_{ij})w_i^{-}+\alpha_{ij}w_j^{-}$ with
$\alpha_{ij}\in(0,1]$, followed by local SGD. The theory below is agnostic to the
merge operator; it concerns \emph{whether, when, and with whom} the single
model-bearing exchange of a decision episode should occur.

\subsection{Transient contacts}
When $j\in\Nset_i(t)$, let $R_{ij}(t)$ be the residual lifetime of the link, with
conditional survival function
\begin{equation}
S_{ij}(\Delta\mid t)=\Prb\bigl(R_{ij}(t)\ge\Delta\mid(i,j)\in\mathcal E(t)\bigr).
\label{eq:survival}
\end{equation}
A full model exchange takes duration $\tau_e$; one lightweight probe takes
$\tau_p\ll\tau_e$. The usefulness of $j$ therefore depends on
$S_{ij}(\tau_e\mid t)$, not only on statistical model quality. The theory
requires only that a survival estimate be available. As a concrete lightweight
estimator, device $i$ maintains cumulative observed contact time $T_{ij}$ and
disconnection count $d_{ij}$ and forms the shrinkage exponential-hazard estimate
\begin{equation}
\widehat h_{ij}=\frac{d_{ij}+a_0}{T_{ij}+b_0},\qquad
\widehat S_{ij}(\Delta)=\exp(-\widehat h_{ij}\Delta),
\label{eq:hazard}
\end{equation}
with $a_0,b_0>0$ stabilising sparse histories.

\subsection{Lightweight, perishable peer evidence}
Each receiver has a small non-private anchor pool $\Aset_i$ (a public calibration
set or generated samples). To probe $j$, device $i$ requests logits or a scalar
score for one anchor mini-batch; the full peer model is not transferred. Let
$X_{ij,r}\in[0,1]$ be the normalised compatibility from probe $r$, with
time-dependent mean $\mu_{ij}(t)=\E[X_{ij,r}\mid t]$. Let $c_e$ be the normalised
exchange cost and $c_p$ the probe cost. The \emph{realisable current value} of an
exchange with $j$ is
\begin{equation}
v_{ij}(t)=S_{ij}(\tau_e\mid t)\,\mu_{ij}(t)-c_e.
\label{eq:realvalue}
\end{equation}
A statistically strong peer may therefore have low practical value when its
residual-contact probability is small. Evidence is perishable along two axes:
$S_{ij}$ decays as contact time is consumed, and $\mu_{ij}(t)$ drifts as the peer
continues local training while disconnected. \Cref{fig:twoclocks} depicts these
two clocks and the interior stopping point they induce.

The multiplicative form \eqref{eq:realvalue} models the realisable value as
statistical compatibility \emph{gated} by the probability that the exchange
completes; it treats link survival and model compatibility as separable, which is
appropriate when link duration is governed by mobility and radio conditions rather
than by model content. Correlated regimes (for example, a peer whose compatibility
and contact stability share a common cause) can be accommodated by replacing the
product with any value functional that is coordinatewise non-decreasing in
$(S_{ij},\mu_{ij})$; all interval-propagation and monotonicity results below use
only this monotonicity, not the product form itself.

\begin{figure}[t]
  \centering
  \includegraphics[width=0.62\linewidth]{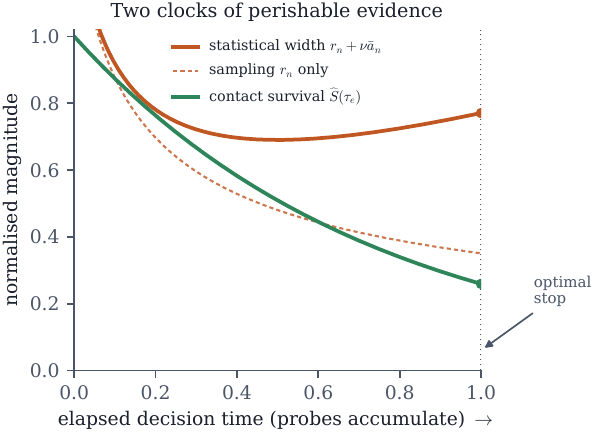}
  \caption{The two clocks of perishable evidence. Statistical uncertainty first
  falls as probes accumulate (sampling half-width $r_n\propto n^{-1/2}$) but then
  rises as cached evidence ages (staleness term $\nu\bar a_n$), while contact
  survival $\widehat S(\tau_e)$ decays monotonically as contact time is consumed.
  The value that actually matters, realisable value, is maximised at an interior
  time, which is the object an optimal stopping rule must locate.}
  \label{fig:twoclocks}
\end{figure}

% ==================================================================== section 4
\section{The Optimal-Stopping Formulation}
\label{sec:formulation}

\subsection{Decision state, actions, and rewards}
A decision episode begins when a collaboration opportunity is due and at least one
contact is available, and runs until a model-bearing action is taken or a deadline
$H$ is reached. At decision stage $n$ (time $t_n\le H$) the receiver observes the
state
\begin{equation}
s_n=\bigl(t_n,\ \Nset_i(t_n),\ \Hset_n,\ \widehat{\bm S}_n,\ \widehat\lambda_i,\ \widehat F_i^V\bigr)\in\Sset,
\label{eq:state}
\end{equation}
where $\Hset_n$ collects the probe observations and their ages, $\widehat{\bm
S}_n$ the residual-contact survival estimates for reachable peers, and
$(\widehat\lambda_i,\widehat F_i^V)$ the receiver's estimated future-contact
arrival rate and value-mark distribution. The admissible actions are
\begin{equation}
\Aset(s)=\{\textsc{Exchange}(j)\}_{j\in\Nset_i}\ \cup\
\{\textsc{Probe}(j)\}_{j:\,n_{ij}<n_{\max}}\ \cup\ \{\textsc{Wait}\}\ \cup\ \{\textsc{Random}\}.
\end{equation}
\textsc{Exchange} and \textsc{Random} are \emph{terminal} (stopping) actions:
they initiate the one model transfer and end the episode. \textsc{Probe} and
\textsc{Wait} are \emph{continuation} actions: they pay a cost and transition to a
new (random) state. Let $V(s)$ denote the optimal expected remaining net value
from state $s$, and let $v_0$ be the deterministic fallback (idle) value obtained
when the episode ends with no exchange. Writing $Q_E,Q_P,Q_W,Q_R$ for the four
action values, the dynamic-programming equation is
\begin{equation}
V(s)=\max\Bigl\{\max_{j}Q_E(s,j),\ \max_{j}Q_P(s,j),\ Q_W(s),\ Q_R(s),\ v_0\Bigr\},
\label{eq:bellman}
\end{equation}
where the terminal option $v_0$ (do nothing) guarantees $V(s)\ge v_0$ and is
distinct from the random-exchange value $Q_R$.
The exchange and probe values are
\begin{align}
Q_E(s,j)&=\widehat S_{ij}(\tau_e)\,\widehat\mu_{ij}-c_e,\label{eq:qE}\\
Q_P(s,j)&=-c_p+q_{ij}\,\E[V(s^+)]+(1-q_{ij})\,\E[V(s^{\mathrm{drop}})],
\qquad q_{ij}=\widehat S_{ij}(\tau_p).\label{eq:qprobe}
\end{align}
Here $s^+$ is the successor state when the probe completes while $j$ remains
connected (a new observation is appended to $\Hset_n$ and the bounds are
refreshed), and $s^{\mathrm{drop}}$ is the successor when $j$ disconnects before
the probe completes: $j$ is removed from the neighbour set, $\Hset_n$ is unchanged
(no completed observation is obtained), and only the decision time and
mobility-dependent quantities are refreshed. The decomposition \eqref{eq:qprobe}
follows from the law of total expectation over the survival event
$A_{ij}=\{R_{ij}(t_n)\ge\tau_p\}$, with $\Prb(A_{ij})=q_{ij}$; the direct cost
$c_p$ is paid in either branch because the probe has already been initiated. The
waiting value $Q_W$ and random value $Q_R$ are specified in
\Cref{sec:wait,sec:policy}.

\subsection{Well-posedness and existence}
Although the state space is high-dimensional and partly continuous, the episode
has a bounded number of continuation steps, which makes backward induction valid.

\begin{proposition}[Finite effective horizon]
\label{prop:horizon}
Suppose each probe consumes at least $\tau_p>0$ of elapsed time and each wait
consumes at least $\Delta_w>0$, and the episode terminates no later than the
deadline $H<\infty$. Then along any trajectory the number of continuation
(\textsc{Probe} or \textsc{Wait}) actions is at most
$\bar n:=\lfloor H/\min(\tau_p,\Delta_w)\rfloor<\infty$, and the number of
probes of any fixed peer is at most $n_{\max}$.
\end{proposition}
\begin{proof}
Each continuation action advances elapsed episode time by at least
$\min(\tau_p,\Delta_w)$, and the episode ends at or before $H$; hence at most
$\lfloor H/\min(\tau_p,\Delta_w)\rfloor$ continuation actions occur. The per-peer
cap is the constraint $n_{ij}<n_{\max}$ in $\Aset(s)$.
\end{proof}

\begin{assumption}[Bounded rewards]
\label{ass:bounded}
Compatibility scores satisfy $X_{ij,r}\in[0,1]$, survival estimates lie in
$[0,1]$, and the costs $c_e,c_p,c_w,\lvert v_0\rvert$ are finite. Consequently
every terminal reward lies in a fixed bounded interval $[\underline g,\bar g]$.
\end{assumption}

\begin{theorem}[Existence and dynamic-programming optimality]
\label{thm:existence}
Under \Cref{prop:horizon} and \Cref{ass:bounded}, the optimal value function $V$
is the unique bounded solution of the finite-horizon Bellman equation
\eqref{eq:bellman} obtained by backward induction on the number of remaining
continuation steps, and there exists a (Markov, deterministic) optimal policy
$\pi^\star$ that in each state selects an action attaining the maximum in
\eqref{eq:bellman}.
\end{theorem}
\begin{proof}
By \Cref{prop:horizon} the number of remaining continuation steps
$k(s)\in\{0,1,\ldots,\bar n\}$ is a well-defined, non-negative, integer potential
that strictly decreases at every continuation action. Define $V^{(0)}(s)$ on
states with $k(s)=0$ (no continuation admissible) as the best terminal reward
$\max\{\max_j Q_E(s,j),Q_R(s),v_0\}$, which is bounded by \Cref{ass:bounded}. For
$k\ge 1$, given $V^{(k-1)}$ on all states reachable in one continuation step,
define
\[
V^{(k)}(s)=\max\Bigl\{\underbrace{\max_j Q_E(s,j),\,Q_R(s),\,v_0}_{\text{terminal}},\ \underbrace{\max_j Q_P(s,j),\,Q_W(s)}_{\text{continuation, using }V^{(k-1)}}\Bigr\},
\]
where the continuation values are the finite expectations in \eqref{eq:qprobe}
and \eqref{eq:wait} taken against $V^{(k-1)}$. Each $V^{(k)}$ is a finite maximum
of bounded quantities and is therefore bounded; the recursion terminates after
$\bar n$ steps because $k$ cannot exceed $\bar n$. Setting $V(s)=V^{(k(s))}(s)$
yields a bounded function satisfying \eqref{eq:bellman}. Uniqueness among bounded
solutions follows because backward induction pins down $V^{(k)}$ from
$V^{(k-1)}$ with no free choices. The maximising action in each state defines a
Markov deterministic policy $\pi^\star$; a standard policy-evaluation argument
(the value of $\pi^\star$ satisfies the same recursion) shows $V^{\pi^\star}=V$,
so $\pi^\star$ is optimal.
\end{proof}

\subsection{Snell-envelope and reservation-value structure}
Separating terminal from continuation actions exposes the optimal-stopping
skeleton of the problem. Define the \emph{best immediate stopping value} and the
\emph{continuation value}
\begin{equation}
g(s)=\max\Bigl\{\max_j Q_E(s,j),\,Q_R(s),\,v_0\Bigr\},\qquad
W(s)=\max\Bigl\{\max_j Q_P(s,j),\,Q_W(s)\Bigr\},
\label{eq:gW}
\end{equation}
where $v_0$ is the deterministic fallback (idle) value obtained when no exchange
is made, so that $V(s)=\max\{g(s),W(s)\}$ by \eqref{eq:bellman}.

\begin{proposition}[Snell envelope and reservation value]
\label{prop:snell}
Under the hypotheses of \Cref{thm:existence}, $V$ is the smallest bounded
function dominating the immediate stopping value and stable under continuation,
i.e.\ $V(s)\ge g(s)$ and $V(s)\ge W(s)$ with equality of the maximum, so $V$ is
the (finite-horizon) Snell envelope of the stopping-reward process. It is optimal
to take a terminal action at the first stage where
\begin{equation}
g(s)\ \ge\ \rho(s):=W(s),
\label{eq:reservation}
\end{equation}
and to continue otherwise. Consequently the optimal policy is a reservation-value
threshold rule: stop and exchange (or take the best terminal option) as soon as
the best realisable value on offer reaches the reservation level $\rho(s)$, which
depends on the current peer only through the continuation.
\end{proposition}
\begin{proof}
By \Cref{thm:existence}, $V=\max\{g,W\}$; hence $V\ge g$, $V\ge W$, and $V$ equals
one of them. Minimality: any bounded $U$ with $U\ge g$ and
$U\ge$ (continuation value computed from $U$) satisfies $U\ge V^{(k)}$ for all $k$
by induction on $k$ (base case $U\ge g=V^{(0)}$; step: $U\ge\max\{g,\text{cont}(U)\}\ge\max\{g,\text{cont}(V^{(k-1)})\}=V^{(k)}$ using monotonicity of the continuation operator in its argument), so $U\ge V$. This is the defining
property of the Snell envelope. Since $V=\max\{g,W\}$, we have $g\ge W\iff
V=g$, i.e.\ stopping is optimal exactly when $g(s)\ge\rho(s)$ with $\rho:=W$.
\end{proof}

The reservation-value view organises the rest of the paper: \Cref{sec:certify}
develops confidence-safe surrogates for the terminal value $g$; \Cref{sec:voi}
and \Cref{sec:wait} develop tractable surrogates and closed forms for the two
components of $\rho=W$ (probing and waiting); and \Cref{sec:myopic} shows when the
resulting one-step rule coincides with the optimal threshold. We record a
stability property used repeatedly.

\begin{lemma}[Monotonicity and $1$-Lipschitzness of the value]
\label{lem:lip}
Fix the continuation dynamics. Then $V$ is non-decreasing in the best immediate
stopping value $g$, and for two states differing only in their terminal values
$g,g'$, $\lvert V(g,\cdot)-V(g',\cdot)\rvert\le\lvert g-g'\rvert$. Moreover
$v_{ij}$ in \eqref{eq:realvalue} is non-decreasing in both $S_{ij}(\tau_e)$ and
$\mu_{ij}$.
\end{lemma}
\begin{proof}
$V=\max\{g,W\}$ with $W$ independent of the current $g$; $x\mapsto\max\{x,c\}$ is
non-decreasing and $1$-Lipschitz, giving the first two claims. The last is
immediate since $S,\mu\ge0$ and the product is coordinatewise non-decreasing on
$[0,1]^2$.
\end{proof}

The next two results record the convexity and threshold structure of the value in
the immediate reward, and the option value of a longer decision horizon; both are
used when reasoning about when stopping becomes optimal.

\begin{proposition}[Convexity and one-sided threshold structure]
\label{prop:threshold}
Fix the continuation dynamics, so that the continuation value $W$ does not depend
on the current best immediate value $g$. Then $g\mapsto V(g,\cdot)=\max\{g,W\}$ is
convex, non-decreasing, and $1$-Lipschitz, and the optimal stopping set in the
$g$-coordinate is the upward-closed half-line $\{g:g\ge\rho\}$ with threshold
$\rho=W$. Consequently, if it is optimal to take a terminal action at some
immediate value $g_0$, it is optimal at every $g\ge g_0$: the optimal policy is a
\emph{monotone} (threshold) rule and never re-enters the continuation region as
the certified value improves.
\end{proposition}
\begin{proof}
$g\mapsto\max\{g,W\}$ is a maximum of two affine functions of $g$, hence convex;
it is non-decreasing and $1$-Lipschitz as in \Cref{lem:lip}. Stopping is optimal
iff $g\ge W$ (\Cref{prop:snell}), and $\{g:g\ge W\}$ is upward closed with
threshold $\rho=W$; upward closure is exactly the stated monotonicity.
\end{proof}

\begin{proposition}[Option value of the decision horizon]
\label{prop:horizonmono}
Let $V^{(k)}$ denote the optimal value with at most $k$ remaining continuation
steps (equivalently, deadline permitting $k$ further probes/waits). Then
$V^{(k)}(s)\ge V^{(k-1)}(s)$ for every $s$ and $k\ge1$; that is, the optimal value
is non-decreasing in the horizon. In particular a later deadline $H$ never lowers
the attainable value.
\end{proposition}
\begin{proof}
By backward induction on $k$. For $k=1$, $V^{(1)}(s)=\max\{g(s),\mathrm{cont}(V^{(0)})(s)\}\ge g(s)=V^{(0)}(s)$, where $\mathrm{cont}(\cdot)$ is the
continuation operator (the finite expectations in \eqref{eq:qprobe} and
\eqref{eq:wait}), which is monotone in its argument because it is a non-negative
mixture of state values minus a fixed cost. Assume $V^{(k-1)}\ge V^{(k-2)}$
pointwise. Monotonicity of $\mathrm{cont}$ gives
$\mathrm{cont}(V^{(k-1)})\ge\mathrm{cont}(V^{(k-2)})$, hence
$V^{(k)}=\max\{g,\mathrm{cont}(V^{(k-1)})\}\ge\max\{g,\mathrm{cont}(V^{(k-2)})\}=V^{(k-1)}$.
Since additional wall-clock time only enlarges the set of admissible continuation
steps, the deadline statement follows.
\end{proof}

% ==================================================================== section 5
\section{Confidence-Safe Certification}
\label{sec:certify}

Exact evaluation of the terminal value $g$ is impossible because $\mu_{ij}(t)$ is
unknown and estimated from a few, ageing probes. We build high-probability
bounds on $\mu_{ij}(t)$, propagate them to realisable value, and derive a
certification rule that commits to a peer only when its realisable value is
provably dominant.

\subsection{Drift-aware confidence bounds}
Let probe $r$ for $(i,j)$ be observed at time $t_{ij,r}$ with score $X_{ij,r}$,
and let its age at decision time $t$ be $a_{ij,r}=t-t_{ij,r}$. Define
\begin{equation}
\widehat\mu_{ij,n}=\frac1n\sum_{r=1}^{n}X_{ij,r},\qquad
\bar a_{ij,n}=\frac1n\sum_{r=1}^{n}a_{ij,r},\qquad
\bar\mu_{ij,n}=\frac1n\sum_{r=1}^{n}\mu_{ij}(t_{ij,r}).
\label{eq:evidence}
\end{equation}
Then $\E[\widehat\mu_{ij,n}]=\bar\mu_{ij,n}$: the empirical mean estimates utility
\emph{at the probe times}, not the current utility $\mu_{ij}(t)$.

\begin{assumption}[Bounded sequential evidence]
\label{ass:seq}
For each active pair $(i,j)$, probe outcomes are conditionally independent given
their observation times and satisfy $X_{ij,r}\in[0,1]$ with mean
$\mu_{ij}(t_{ij,r})$.
\end{assumption}

\begin{assumption}[Local utility drift]
\label{ass:drift}
Over one evidence-retention horizon,
$\lvert\mu_{ij}(t)-\mu_{ij}(t')\rvert\le\nu_{ij}\lvert t-t'\rvert$ for a finite
drift envelope $\nu_{ij}\ge0$.
\end{assumption}

\begin{lemma}[Finite-horizon stage-uniform concentration]
\label{lem:uniform}
Fix a receiver that tracks at most $K$ distinct peer identities during a finite
decision episode and at most $n_{\max}$ probes per peer. Under \Cref{ass:seq},
with probability at least $1-\delta$,
\begin{equation}
\lvert\widehat\mu_{ij,n}-\bar\mu_{ij,n}\rvert\le r_n(\delta):=\sqrt{\frac{\log(2Kn_{\max}/\delta)}{2n}}
\label{eq:radius}
\end{equation}
simultaneously for every tracked peer $j$ and every $1\le n\le n_{\max}$.
\end{lemma}
\begin{proof}
For fixed $(j,n)$, Hoeffding's inequality \citep{hoeffding1963} gives
$\Prb(\lvert\widehat\mu_{ij,n}-\bar\mu_{ij,n}\rvert>r_n)\le
2e^{-2nr_n^2}$. Choosing $r_n$ so that this equals $\delta/(Kn_{\max})$ yields
\eqref{eq:radius}. A union bound over the at most $Kn_{\max}$ peer--stage pairs
gives the simultaneous statement.
\end{proof}

Because the band holds simultaneously for all $1\le n\le n_{\max}$, it holds in
particular at any \emph{data-dependent} probe count: the number of probes PROSE
takes of a peer is a stopping time adapted to $\{\mathcal F_n\}$, and the uniform
guarantee is exactly what licenses its use at such an adaptively chosen $n$.

\begin{lemma}[Validity under stale evidence]
\label{lem:stale}
Under \Cref{ass:drift}, on the event of \Cref{lem:uniform} the current utility
satisfies $\mu_{ij}(t)\in[L_{ij,n}(t),U_{ij,n}(t)]$ for every tracked
peer--stage pair, where
\begin{align}
L_{ij,n}(t)&=\clip\bigl(\widehat\mu_{ij,n}-r_n-\nu_{ij}\bar a_{ij,n},\,0,\,1\bigr),
\label{eq:driftlower}\\
U_{ij,n}(t)&=\clip\bigl(\widehat\mu_{ij,n}+r_n+\nu_{ij}\bar a_{ij,n},\,0,\,1\bigr).
\label{eq:driftci}
\end{align}
\end{lemma}
\begin{proof}
By \Cref{ass:drift}, $\lvert\mu_{ij}(t)-\mu_{ij}(t_{ij,r})\rvert\le\nu_{ij}a_{ij,r}$;
averaging over $r$ gives
$\lvert\mu_{ij}(t)-\bar\mu_{ij,n}\rvert\le\nu_{ij}\bar a_{ij,n}$. Combining with
\Cref{lem:uniform} by the triangle inequality,
$\lvert\mu_{ij}(t)-\widehat\mu_{ij,n}\rvert\le r_n+\nu_{ij}\bar a_{ij,n}$, and
intersecting with $[0,1]$ (since $\mu_{ij}\in[0,1]$) gives
\eqref{eq:driftlower}--\eqref{eq:driftci}.
\end{proof}

The interval width has two sources: $r_n$ shrinks with additional probes, whereas
$\nu_{ij}\bar a_{ij,n}$ grows as evidence ages. A newly encountered peer
($n_{ij}=0$) uses the vacuous bounded-score interval $[0,1]$ until its first probe
completes. \Cref{fig:certify}(left) decomposes one interval. The following lemma
makes the two-source tension precise and is used later to establish diminishing
returns of probing.

\begin{lemma}[Sampling contraction versus staleness inflation]
\label{lem:width}
Fix $\delta,K,n_{\max}$. The sampling radius $r_n(\delta)$ in \eqref{eq:radius} is
strictly decreasing in $n$ with $r_n(\delta)=\Theta(n^{-1/2})$. Consequently, in
the drift-free case $\nu_{ij}=0$ the compatibility intervals are nested,
$[L_{ij,n+1},U_{ij,n+1}]\subseteq[L_{ij,n},U_{ij,n}]$ up to re-centring, and their
width $2r_n$ contracts monotonically to $0$. With drift, the width
$w_{ij,n}=r_n+\nu_{ij}\bar a_{ij,n}$ is a sum of a strictly decreasing term and a
term non-decreasing in the mean evidence age; if the age grows without bound the
width is eventually increasing, so a finite retention horizon is necessary for the
width to remain controlled.
\end{lemma}
\begin{proof}
Write $r_n(\delta)=\sqrt{c/n}$ with $c=\tfrac12\log(2Kn_{\max}/\delta)>0$ fixed;
$n\mapsto\sqrt{c/n}$ is strictly decreasing and $\Theta(n^{-1/2})$, and tends to
$0$. In the drift-free case the half-width is exactly $r_n$, so widths contract
and successive intervals nest up to the shift of centre $\widehat\mu_{ij,n}$. In
general $w_{ij,n}=r_n+\nu_{ij}\bar a_{ij,n}$; the first term decreases in $n$ and
the second is non-decreasing in $\bar a_{ij,n}$, giving the stated tension. If
$\bar a_{ij,n}\to\infty$ then $\nu_{ij}\bar a_{ij,n}\to\infty$ dominates the
vanishing $r_n$, so $w_{ij,n}\to\infty$; hence some finite retention horizon is
required.
\end{proof}

\subsection{Mobility-adjusted value bounds and certification}
The interval above concerns statistical compatibility. What matters is realisable
value \eqref{eq:realvalue}. Suppose survival intervals are available:
$S_{ij}(\tau_e)\in[\underline S_{ij}(\tau_e),\overline S_{ij}(\tau_e)]$.

\begin{assumption}[Survival-interval coverage]
\label{ass:surv}
With probability at least $1-\delta_S$, simultaneously for all tracked peers,
$S_{ij}(\tau_e)\in[\underline S_{ij}(\tau_e),\overline S_{ij}(\tau_e)]$. (With
point survival estimates one sets $\underline S=\overline S=\widehat S$; the
high-probability certification statement then holds conditionally on the survival
estimates being exact.)
\end{assumption}

Because both factors are non-negative and $S\mu$ is coordinatewise monotone on
$[0,1]^2$, interval propagation gives realisable-value bounds
\begin{equation}
\underline V_{ij,n}=\underline S_{ij}(\tau_e)\,L_{ij,n}-c_e,\qquad
\overline V_{ij,n}=\overline S_{ij}(\tau_e)\,U_{ij,n}-c_e.
\label{eq:valuebounds}
\end{equation}
Among reachable peers, PROSE selects as \emph{incumbent} the peer with the
largest \emph{worst-case} realisable value,
\begin{equation}
b=\argmax_{j\in\Nset_i(t)}\underline V_{ij,n},
\label{eq:incumbent}
\end{equation}
a maximin choice: it selects the peer with the largest value already guaranteed by
the current confidence set, not merely the largest estimated mean. Let $v_0$ be
the deterministic fallback (idle) value of \eqref{eq:bellman}. The
\emph{certification rule} accepts $b$ when
\begin{equation}
\underline V_{ib,n}>\max\Bigl\{v_0,\ \max_{k\neq b}\overline V_{ik,n}\Bigr\}.
\label{eq:certify}
\end{equation}

\begin{figure}[t]
  \centering
  \includegraphics[width=\linewidth]{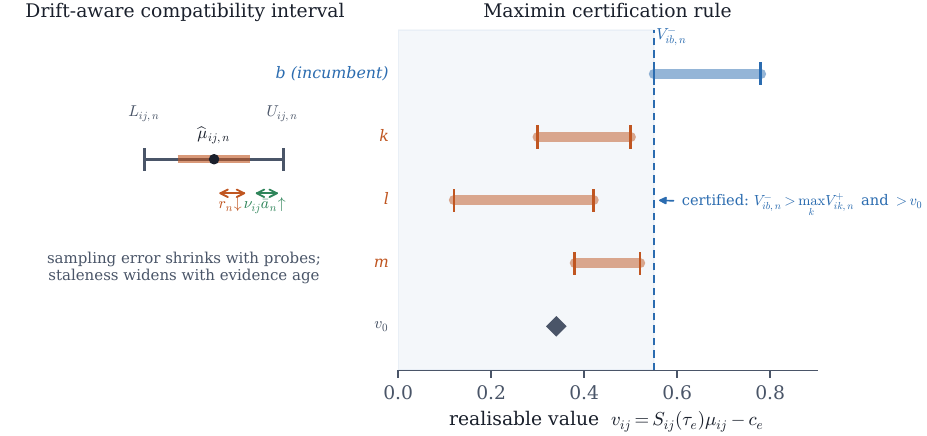}
  \caption{Confidence-safe certification. \emph{Left:} a drift-aware interval
  $[L_{ij,n},U_{ij,n}]$ centred at $\widehat\mu_{ij,n}$ with half-width
  $r_n+\nu_{ij}\bar a_n$; the sampling part shrinks with probes while the
  staleness part grows with age. \emph{Right:} realisable-value intervals across
  peers; the incumbent $b$ is certified because its lower bound
  $\underline V_{ib,n}$ clears the upper bound of every rival and the fallback
  $v_0$ (\Cref{eq:certify}).}
  \label{fig:certify}
\end{figure}

\begin{theorem}[Correct mobility-adjusted certification]
\label{thm:cert}
On the joint event of \Cref{lem:stale} and \Cref{ass:surv}, if peer $b$ satisfies
\eqref{eq:certify} then
\begin{equation}
v_{ib}>\max\Bigl\{v_0,\ \max_{k\neq b}v_{ik}\Bigr\}.
\end{equation}
Consequently certification is correct with probability at least
$1-\delta-\delta_S$. The conclusion concerns the current exchange/fallback set
only; it does not assert dominance over the waiting action, which is compared
separately through $\rho=W$ (\Cref{prop:snell}).
\end{theorem}
\begin{proof}
By \Cref{lem:lip} and non-negativity, interval multiplication is monotone, so on
the joint event $v_{ij}\in[\underline V_{ij,n},\overline V_{ij,n}]$ for all $j$.
Then \eqref{eq:certify} gives, for every $k\neq b$,
$v_{ib}\ge\underline V_{ib,n}>\overline V_{ik,n}\ge v_{ik}$, and
$v_{ib}\ge\underline V_{ib,n}>v_0$. The probability bound is a union bound over
the evidence event ($\ge1-\delta$) and the survival event ($\ge1-\delta_S$).
\end{proof}

\begin{theorem}[Finite-sample current-peer identification]
\label{thm:samples}
Assume exact survival probabilities, no drift during one episode, equal probe
counts $n$, and a unique best current peer $b$ with mobility-adjusted gap
\begin{equation}
\Delta_v=v_{ib}-\max\Bigl\{v_0,\ \max_{k\neq b}v_{ik}\Bigr\}>0.
\end{equation}
Let $S_{\max}=\max_j S_{ij}(\tau_e)$. On the event of \Cref{lem:uniform}, a
sufficient condition for \eqref{eq:certify} to hold (so that $b$ is certified
against every current peer and the fallback) is
\begin{equation}
n>\frac{8S_{\max}^2}{\Delta_v^2}\,\log\frac{2Kn_{\max}}{\delta}.
\label{eq:samplebound}
\end{equation}
\end{theorem}
\begin{proof}
With no drift and exact survival, the realisable-value interval of each peer has
full width at most $2S_{\max}r_n$, and its true value lies anywhere inside. Hence
the incumbent's lower bound can sit a full width below its truth,
$\underline V_{ib,n}\ge v_{ib}-2S_{\max}r_n$, and each competitor's upper bound a
full width above its truth, $\overline V_{ik,n}\le v_{ik}+2S_{\max}r_n$.
Therefore $\underline V_{ib,n}-\overline V_{ik,n}\ge(v_{ib}-v_{ik})-4S_{\max}r_n
\ge\Delta_v-4S_{\max}r_n$, so $\underline V_{ib,n}>\overline V_{ik,n}$ for every
$k\neq b$ whenever $4S_{\max}r_n<\Delta_v$. Against the deterministic fallback only
the incumbent's single-sided deviation matters,
$\underline V_{ib,n}\ge v_{ib}-2S_{\max}r_n>v_0$ when $2S_{\max}r_n<\Delta_v$,
which the same condition implies. (The condition $4S_{\max}r_n<\Delta_v$ also
forces $b$ to be the maximin incumbent, since its lower bound then exceeds every
rival's upper and hence lower bound.) Substituting \eqref{eq:radius} into
$4S_{\max}r_n<\Delta_v$, squaring, and rearranging gives \eqref{eq:samplebound}.
\end{proof}

\begin{corollary}[Eventual certification]
\label{cor:consistency}
Under the hypotheses of \Cref{thm:samples} with $\Delta_v>0$, certification occurs
after finitely many probes on the event of \Cref{lem:uniform}: the required probe
count is the finite right-hand side of \eqref{eq:samplebound}. Hence the
confidence-safe rule is consistent whenever a strict realisable-value gap exists
and the contact survives long enough to collect the probes.
\end{corollary}

\subsection{Regret of the maximin selection}
Even before certification, the maximin incumbent \eqref{eq:incumbent} carries a
one-sided guarantee: committing to it can lose at most one interval width relative
to the (unknown) best current peer. Write
$w^{V}_{ij,n}=\overline V_{ij,n}-\underline V_{ij,n}\ge0$ for the realisable-value
interval width (the value-space counterpart of the compatibility half-width
$w_{ij,n}$ of \Cref{lem:width}, with $w^{V}_{ij,n}\le2S_{\max}w_{ij,n}$).

\begin{theorem}[Maximin selection regret]
\label{thm:regret}
On the joint event of \Cref{lem:stale} and \Cref{ass:surv}, let
$j^\star=\argmax_{j\in\Nset_i(t)}v_{ij}$ be the (unknown) best current peer and let
$b$ be the maximin incumbent \eqref{eq:incumbent}. Then the instantaneous
selection regret of exchanging with $b$ satisfies
\begin{equation}
v_{ij^\star}-v_{ib}\ \le\ w^{V}_{ij^\star,n}\ \le\ \max_{j\in\Nset_i(t)}w^{V}_{ij,n}.
\label{eq:regret}
\end{equation}
That is, the maximin choice is within the widest realisable-value confidence
interval of the best peer.
\end{theorem}
\begin{proof}
On the stated event $v_{ij}\in[\underline V_{ij,n},\overline V_{ij,n}]$ for all
$j$. Hence
\[
v_{ij^\star}\le\overline V_{ij^\star,n}
=\underline V_{ij^\star,n}+w^{V}_{ij^\star,n}
\le\underline V_{ib,n}+w^{V}_{ij^\star,n}
\le v_{ib}+w^{V}_{ij^\star,n},
\]
where the middle inequality uses $\underline V_{ij^\star,n}\le\underline V_{ib,n}$
(the incumbent maximises the lower bound, \eqref{eq:incumbent}) and the last uses
$\underline V_{ib,n}\le v_{ib}$. Rearranging gives the first inequality in
\eqref{eq:regret}; the second is immediate.
\end{proof}

\begin{corollary}[Vanishing regret]
\label{cor:vanish}
Assume exact survival and equal probe counts $n$, and let $S_{\max}=\max_j
S_{ij}(\tau_e)$. Then $w^{V}_{ij,n}\le 2S_{\max}\bigl(r_n+\nu_{ij}\bar a_{ij,n}\bigr)$
and, in the drift-free case, $v_{ij^\star}-v_{ib}\le 2S_{\max}r_n
=O\!\bigl(\sqrt{n^{-1}\log(Kn_{\max}/\delta)}\bigr)\to0$. Thus maximin selection is
asymptotically oracle-optimal at the rate of \Cref{lem:width}.
\end{corollary}
\begin{proof}
By \eqref{eq:valuebounds} and \Cref{lem:stale},
$w^{V}_{ij,n}=\overline S_{ij}U_{ij,n}-\underline S_{ij}L_{ij,n}\le
S_{\max}(U_{ij,n}-L_{ij,n})\le 2S_{\max}(r_n+\nu_{ij}\bar a_{ij,n})$ under exact
survival $\underline S=\overline S\le S_{\max}$. Setting $\nu_{ij}=0$ and
substituting into \eqref{eq:regret} with \Cref{lem:width} gives the rate.
\end{proof}

\subsection{An interior optimal stopping time}
The two clocks of \Cref{fig:twoclocks} induce an interior optimum: the guaranteed
realisable value first rises as sampling error contracts and then falls as
survival decays and evidence goes stale. We formalise existence.

\begin{proposition}[Existence of an interior optimal stopping time]
\label{prop:interior}
Model the guaranteed realisable value of a tracked peer as a function of elapsed
episode time $t\ge0$ by $\underline V(t)=S(t)L(t)-c_e$, where $S:[0,\infty)\to(0,1]$
is continuous and strictly decreasing with $S(0)=1$ and $S(t)\to0$, and
$L:[0,\infty)\to[0,1]$ is continuous with $L(0)=0$ (a newly encountered peer has
the vacuous lower bound) and $\sup_t L(t)>0$. Then $\underline V$ attains its
maximum at some finite time $t^\star\in(0,\infty)$, and $\underline V(t^\star)>
\underline V(0)=\underline V(\infty^-)=-c_e$. If in addition the region
$\{t:L(t)>0\}$ is an interval on which $\log S$ is concave and $\log L$ is strictly
concave, then $t^\star$ is unique.
\end{proposition}
\begin{proof}
$\underline V$ is continuous with $\underline V(0)=S(0)\cdot0-c_e=-c_e$ and
$\lim_{t\to\infty}\underline V(t)=-c_e$ because $S(t)\to0$ and $L$ is bounded.
Since $\sup_tL(t)>0$ and $S>0$, there is $t_1$ with
$\underline V(t_1)=S(t_1)L(t_1)-c_e>-c_e$; set
$M=\sup_t\underline V(t)\ge\underline V(t_1)>-c_e$. Choose $\varepsilon\in(0,M+c_e)$
and $T<\infty$ with $\underline V(t)<M-\varepsilon$ for all $t\ge T$ (possible as
$\underline V(t)\to-c_e<M-\varepsilon$ for small $\varepsilon$); by the extreme
value theorem $\underline V$ attains its maximum over the compact set $[0,T]$, and
this value equals $M$ since values beyond $T$ are strictly smaller. The maximiser
$t^\star$ is interior because the boundary values $\underline V(0)=-c_e<M$ and,
for $t\ge T$, $\underline V(t)<M$. For uniqueness, on $\{L>0\}$ write
$\log\underline V^+(t)$ where $\underline V^+=S L$; then $\log(SL)=\log S+\log L$
is strictly concave (a concave plus a strictly concave function), so $SL$ is
strictly log-concave and hence strictly quasi-concave, giving a unique maximiser
of $SL$ and therefore of $\underline V=SL-c_e$.
\end{proof}

% ==================================================================== section 6
\section{Mobility-Aware Value of Information and Comparative Statics}
\label{sec:voi}

Certification decides \emph{whether} the best current peer is provably dominant.
The complementary question is \emph{whether another probe is worth taking}, given
that probing consumes contact lifetime.

\subsection{A sufficient do-not-probe rule}
Let
\begin{equation}
B_n=\max\Bigl\{v_0,\ \max_j\underline V_{ij,n}\Bigr\},\qquad
U_n^{\max}=\max_j\overline V_{ij,n},\qquad
\VPI_n^{\mathrm{ub}}=\bigl[U_n^{\max}-B_n\bigr]_+ .
\label{eq:vpi}
\end{equation}
$B_n$ is the value already guaranteed; $U_n^{\max}$ is the largest value still
compatible with the current intervals; and $\VPI_n^{\mathrm{ub}}$ upper-bounds the
improvement in guaranteed value achievable even with \emph{perfect} information
about current candidates. Mobility adds exposure: while a probe of duration
$\tau_p$ runs, the incumbent $b$ disappears with probability
$1-\widehat S_{ib}(\tau_p)$, placing its protected surplus at risk. Define the
opportunity-risk proxy
\begin{equation}
\Omega_j=\bigl(1-\widehat S_{ij}(\tau_p)\bigr)\bigl[\underline V_{ij,n}-v_0\bigr]_+,
\label{eq:opp}
\end{equation}
and the mobility-aware \emph{do-not-probe} test
$\VPI_n^{\mathrm{ub}}\le c_p+\Omega_b$.

\begin{theorem}[Sufficient mobility-aware stopping rule]
\label{thm:vpi}
Let $B_n,U_n^{\max},\VPI_n^{\mathrm{ub}}$ be as in \eqref{eq:vpi} and $\Omega_b$
as in \eqref{eq:opp}. Suppose the only benefit of another probe is to resolve
uncertainty among currently available peers, and the actual loss caused by
incumbent disappearance during that probe is lower-bounded by $\Omega_b$. If
\begin{equation}
\VPI_n^{\mathrm{ub}}\le c_p+\Omega_b,
\label{eq:noprobe}
\end{equation}
then no single additional probe can improve the current guaranteed value after
accounting for probe cost and opportunity loss. Without the opportunity-loss
lower-bound assumption, $\VPI_n^{\mathrm{ub}}\le c_p$ remains a sufficient
do-not-probe condition under the same current-candidate premise.
\end{theorem}
\begin{proof}
Perfect information about all current candidates raises the guaranteed value by at
most $U_n^{\max}-B_n=\VPI_n^{\mathrm{ub}}$, and one actual probe cannot exceed
perfect information. A probe costs $c_p$ and, by hypothesis, exposes at least
$\Omega_b$ of protected surplus to disconnection. If the upper bound on gain does
not exceed $c_p+\Omega_b$, the net guaranteed gain of probing is non-positive.
Setting $\Omega_b=0$ gives the assumption-light form.
\end{proof}

When probing remains admissible, PROSE probes the peer with the largest
mobility-adjusted acquisition index
\begin{equation}
I_{ij}=\widehat S_{ij}(\tau_p)\,\bigl(\overline V_{ij,n}-\underline V_{ij,n}\bigr)-c_p-\Omega_j,
\label{eq:index}
\end{equation}
among peers with $I_{ij}>0$ and $n_{ij}<n_{\max}$: it rewards uncertainty likely
to survive the probe and charges acquisition and opportunity risk.

The value-of-information upper bound is controlled by the widest realisable-value
interval, which by \Cref{lem:width} contracts as probing proceeds. This yields a
deterministic diminishing-returns envelope that the myopic-optimality analysis of
\Cref{sec:myopic} relies upon.

\begin{lemma}[Diminishing returns of probing]
\label{lem:diminish}
At every stage, $\VPI_n^{\mathrm{ub}}\le\max_{j\in\Nset_i(t)}w^{V}_{ij,n}$, where
$w^{V}_{ij,n}=\overline V_{ij,n}-\underline V_{ij,n}$. Hence, absent new peer arrivals
and with exact survival and equal drift-free probing, $\VPI_n^{\mathrm{ub}}\le
2S_{\max}r_n$ is deterministically non-increasing in $n$ and tends to $0$.
\end{lemma}
\begin{proof}
Let $j'\in\argmax_j\overline V_{ij,n}$, so $U_n^{\max}=\overline V_{ij',n}$. By
\eqref{eq:vpi}, $B_n\ge\underline V_{ij',n}$, hence
$U_n^{\max}-B_n\le\overline V_{ij',n}-\underline V_{ij',n}=w^{V}_{ij',n}\le\max_j
w^{V}_{ij,n}$, and $\VPI_n^{\mathrm{ub}}=[U_n^{\max}-B_n]_+$ inherits the bound. Under
exact survival and no drift, $w^{V}_{ij,n}\le2S_{\max}r_n$ (as in
\Cref{cor:vanish}), which is strictly decreasing in $n$ by \Cref{lem:width} and
vanishes.
\end{proof}

\subsection{Comparative statics of mobility}
The next results make precise the intuition that volatility should curtail
information-gathering. \Cref{fig:regions} visualises both the decision regions and
the monotone decay of the probe value.

\begin{theorem}[Higher link hazard lowers continuation value]
\label{thm:hazard}
Consider a one-step change in the exponential residual-contact hazard $h$ of a
current peer, with probe-survival $q=e^{-h\tau_p}$. Let $C=\E[V(s^+)]$ and
$D=\E[V(s^{\mathrm{drop}})]$ be the completed-probe and drop continuations, with
$C\ge D$, held fixed with respect to this change. Then
$Q_P(h)=-c_p+e^{-h\tau_p}C+(1-e^{-h\tau_p})D$ is non-increasing in $h$, and
strictly decreasing when $C>D$ and $\tau_p>0$. Against any stopping action whose
value is unaffected by this perturbation, the one-step stopping region therefore
expands as $h$ increases.
\end{theorem}
\begin{proof}
Differentiating, $\dfrac{dQ_P}{dh}=-\tau_p e^{-h\tau_p}(C-D)\le0$, with strict
inequality when $C>D$ and $\tau_p>0$. Since $Q_P$ falls while the competing
stopping value is unchanged, the set of hazards on which stopping dominates
probing grows.
\end{proof}

\begin{proposition}[Acquisition index is monotone in hazard]
\label{prop:indexmono}
Under the exponential model $\widehat S_{ij}(\tau_p)=e^{-h_{ij}\tau_p}$, the
acquisition index \eqref{eq:index} is non-increasing in the hazard $h_{ij}$,
holding the value bounds and $v_0$ fixed. Hence higher hazard weakly shrinks the
probe-eligible set $\{j:I_{ij}>0\}$.
\end{proposition}
\begin{proof}
Write $q=e^{-h_{ij}\tau_p}$, decreasing in $h_{ij}$, and let
$\omega=[\underline V_{ij,n}-v_0]_+\ge0$ and $\Delta U=\overline
V_{ij,n}-\underline V_{ij,n}\ge0$. Then
$I_{ij}=q\,\Delta U-c_p-(1-q)\omega=q(\Delta U+\omega)-c_p-\omega$, which is
non-decreasing in $q$ and thus non-increasing in $h_{ij}$. Monotonicity of the
threshold-crossing $\{I_{ij}>0\}$ follows.
\end{proof}

\begin{figure}[t]
  \centering
  \includegraphics[width=\linewidth]{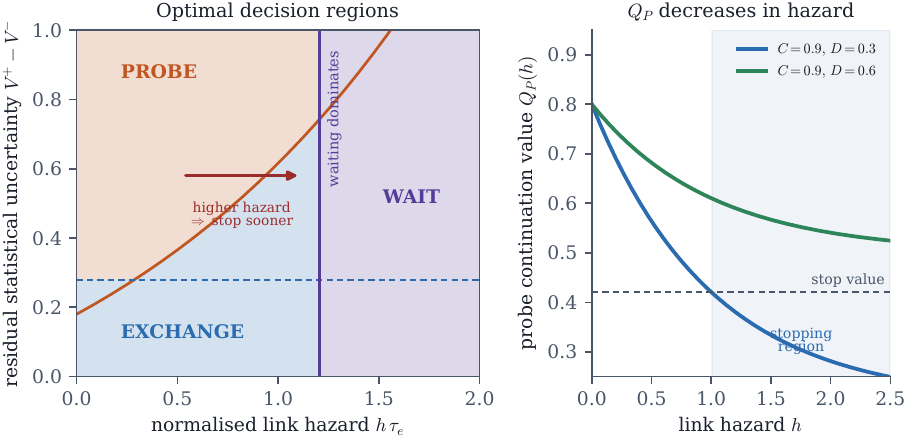}
  \caption{\emph{Left:} optimal decision regions in the plane of normalised link
  hazard $h\tau_e$ and residual statistical uncertainty
  $\overline V-\underline V$. Low uncertainty and adequate survival certify an
  \textsc{Exchange}; high uncertainty with survivable links favours a
  \textsc{Probe}; high hazard makes current links doomed and the future-arrival
  option dominant (\textsc{Wait}). Increasing hazard shifts the boundary so the
  learner stops sooner. \emph{Right:} the probe continuation value $Q_P(h)$
  decreases monotonically in hazard (\Cref{thm:hazard}); once it falls below a
  hazard-independent stopping value, the stopping region is entered.}
  \label{fig:regions}
\end{figure}

% ==================================================================== section 7
\section{The Waiting Option and Reservation Values}
\label{sec:wait}

Waiting exercises an option on future contacts. Let new contacts arrive according
to a marked-Poisson process of rate $\lambda_i$, each with realisable value mark
$V\sim F_i^V$ supported on $[v_{\min},v_{\max}]$, estimated from recent contact
history. \Cref{fig:waiting} plots the resulting value and reservation levels.

\begin{proposition}[Value of waiting under marked-Poisson arrivals]
\label{prop:wait}
Suppose future contacts during a window $\Delta$ form a Poisson process of rate
$\lambda$ and each arrival has an independent mark $V\sim F^V$ on
$[v_{\min},v_{\max}]$. Let $M_\Delta$ be the maximum arrival value
($M_\Delta=-\infty$ if none arrives) and $v_0\in[v_{\min},v_{\max}]$. Then, with a
delay price $c_w$,
\begin{equation}
Q_W(v_0,\Delta)=v_0+\int_{v_0}^{v_{\max}}\bigl[1-e^{-\lambda\Delta(1-F^V(x))}\bigr]\,dx-c_w\Delta.
\label{eq:wait}
\end{equation}
\end{proposition}
\begin{proof}
The number of arrivals with marks exceeding $x$ is Poisson with mean
$\lambda\Delta(1-F^V(x))$, so
$\Prb(M_\Delta\le x)=e^{-\lambda\Delta(1-F^V(x))}$. By the tail-integral
identity, for $Y=\max\{v_0,M_\Delta\}$,
$\E[Y]-v_0=\int_{v_0}^{v_{\max}}\Prb(Y>x)\,dx=\int_{v_0}^{v_{\max}}[1-\Prb(M_\Delta\le x)]\,dx$,
which is the integral in \eqref{eq:wait}. Subtracting the delay cost $c_w\Delta$
gives the claim.
\end{proof}

\begin{lemma}[Monotonicity of the waiting value]
\label{lem:qwmono}
The gross waiting value $Q_W(v_0,\Delta)+c_w\Delta$ is non-decreasing in the
arrival rate $\lambda$, in the wait window $\Delta$, and under a first-order
stochastic improvement of the mark law $F^V$. In particular a busier or
better-mark environment makes waiting weakly more attractive, and the net value
$Q_W$ is quasi-concave in $\Delta$ with an interior optimal window whenever the
marginal arrival benefit at $\Delta=0$ exceeds $c_w$.
\end{lemma}
\begin{proof}
Write the gross value as
$G(\lambda,\Delta)=v_0+\int_{v_0}^{v_{\max}}\bigl[1-e^{-\lambda\Delta(1-F^V(x))}\bigr]dx$.
For each $x$ the integrand $1-e^{-\lambda\Delta(1-F^V(x))}$ is non-decreasing in
$\lambda$ and in $\Delta$ (its exponent's magnitude grows), so $G$ is
non-decreasing in both; a first-order improvement lowers $F^V$ pointwise, raising
each integrand, hence $G$. For quasi-concavity of $Q_W=G-c_w\Delta$ in $\Delta$,
$\partial_\Delta G=\int_{v_0}^{v_{\max}}\lambda(1-F^V(x))e^{-\lambda\Delta(1-F^V(x))}dx$
is positive and strictly decreasing in $\Delta$ (each summand decreases), so
$\partial_\Delta Q_W=\partial_\Delta G-c_w$ crosses zero at most once from above;
it does so in the interior iff $\partial_\Delta G|_{\Delta=0}=\lambda\int_{v_0}^{v_{\max}}(1-F^V(x))dx>c_w$.
\end{proof}

For a persistently mobile receiver the natural object is not a fixed window but a
stationary acceptance threshold. The following result casts waiting as a
search problem with recall and yields a reservation value that separates
acceptance from continued waiting.

\begin{theorem}[Reservation value of the waiting option]
\label{thm:reservation}
Consider a receiver that, with full recall, may either exchange at its current
best realisable value $b$ or continue waiting for marked-Poisson arrivals (rate
$\lambda$, i.i.d.\ marks $V\sim F^V$ with density $f>0$ on $[v_{\min},v_{\max}]$)
at delay price $c_w$ per unit time. Define
$T(\rho)=\int_{\rho}^{v_{\max}}(1-F^V(x))\,dx$. If
$0<c_w<\lambda\,T(v_{\min})=\lambda\,\E[V-v_{\min}]$, then the optimal policy
within stationary threshold policies is to exchange iff $b\ge\rho^\star$, where
$\rho^\star\in(v_{\min},v_{\max})$ is the unique root of
\begin{equation}
\lambda\int_{\rho^\star}^{v_{\max}}\bigl(1-F^V(x)\bigr)\,dx=c_w.
\label{eq:reseqn}
\end{equation}
Moreover the reservation value equals the optimal continuation value,
$\rho^\star=J(\rho^\star)$, where $J(\rho)$ is the expected net payoff of the
threshold-$\rho$ policy. If instead $c_w\ge\lambda\,T(v_{\min})$, waiting is never
worthwhile and $\rho^\star=v_{\min}$.
\end{theorem}
\begin{proof}
Under the threshold-$\rho$ policy the acceptable-offer stream (marks $\ge\rho$) is
Poisson with rate $\lambda(1-F^V(\rho))$, so the expected waiting time to
acceptance is $1/[\lambda(1-F^V(\rho))]$ and the expected accepted value is
$\E[V\mid V\ge\rho]=\rho+T(\rho)/(1-F^V(\rho))$. Hence the expected net payoff is
\[
J(\rho)=\rho+\frac{T(\rho)-c_w/\lambda}{1-F^V(\rho)} .
\]
Writing $Q=1-F^V(\rho)$ (so $Q'=-f(\rho)$) and $T'(\rho)=-(1-F^V(\rho))=-Q$,
\[
J'(\rho)=1+\frac{T'(\rho)\,Q-\bigl(T(\rho)-c_w/\lambda\bigr)Q'}{Q^2}
=1+\frac{-Q^2+f(\rho)\bigl(T(\rho)-c_w/\lambda\bigr)}{Q^2}
=\frac{f(\rho)\bigl(T(\rho)-c_w/\lambda\bigr)}{Q^2}.
\]
Since $f>0$ and $Q>0$, $J'(\rho)$ has the sign of $T(\rho)-c_w/\lambda$. Because
$T$ is continuous and strictly decreasing from $T(v_{\min})=\E[V-v_{\min}]$ to
$T(v_{\max})=0$, the condition $c_w/\lambda\in(0,T(v_{\min}))$ guarantees a unique
$\rho^\star$ with $T(\rho^\star)=c_w/\lambda$, and $J'>0$ for $\rho<\rho^\star$,
$J'<0$ for $\rho>\rho^\star$; thus $J$ is maximised at $\rho^\star$, which
therefore is the optimal stationary threshold. At $\rho=\rho^\star$ the numerator
$T(\rho^\star)-c_w/\lambda=0$, so $J(\rho^\star)=\rho^\star$: the reservation
value coincides with the value of the option, and it is optimal to exchange iff
$b\ge\rho^\star$. If $c_w/\lambda\ge T(v_{\min})$, then $T(\rho)-c_w/\lambda\le0$
for all $\rho$, so $J'\le0$ everywhere and the optimum is at $\rho^\star=v_{\min}$,
i.e.\ accept immediately.
\end{proof}

\Cref{thm:reservation} optimises over stationary threshold policies. Because the
arrival process is memoryless and marks are i.i.d.\ with recall, the restriction
is without loss: the reservation policy is optimal among \emph{all} policies.

\begin{lemma}[Global optimality of the reservation policy]
\label{lem:global}
In the search problem of \Cref{thm:reservation}, the policy ``exchange iff the
available value is at least $\rho^\star$'' is optimal within the class of all
(history-dependent, non-anticipating) policies, and the value of optimally
continuing to wait is the constant $\rho^\star$, independent of calendar time and
history.
\end{lemma}
\begin{proof}
Since arrivals are Poisson and marks i.i.d.\ with full recall, the receiver's
continuation problem after any history in which it has not yet exchanged is a
verbatim copy of the original problem: the residual arrival process is again
Poisson$(\lambda)$ by the memoryless property, and past rejected offers are
irrelevant except through recall, which only raises the current best. Hence the
value of optimally continuing to wait is a constant $\rho_c$ independent of time
and history. The Bellman optimality equation for an available value $x$ reads
$\max\{x,\rho_c\}$, so it is optimal to exchange iff $x\ge\rho_c$; this is a
threshold policy, and by \Cref{thm:reservation} the best threshold is
$\rho^\star$ with value $J(\rho^\star)=\rho^\star$. Therefore $\rho_c=\rho^\star$
and the reservation policy attains the optimum over all policies.
\end{proof}

\begin{proposition}[Comparative statics of the reservation value]
\label{prop:reservationcs}
Let $\rho^\star$ solve \eqref{eq:reseqn} in the interior regime. Then
$\rho^\star$ is strictly increasing in the arrival rate $\lambda$, strictly
decreasing in the delay price $c_w$, and non-decreasing under a first-order
stochastic improvement of the mark distribution $F^V$.
\end{proposition}
\begin{proof}
Implicit differentiation of $\lambda T(\rho^\star)=c_w$, using
$T'(\rho^\star)=-(1-F^V(\rho^\star))<0$, gives
\[
\frac{d\rho^\star}{d\lambda}=\frac{-T(\rho^\star)}{\lambda T'(\rho^\star)}
=\frac{T(\rho^\star)}{\lambda\,(1-F^V(\rho^\star))}>0,\qquad
\frac{d\rho^\star}{dc_w}=\frac{1}{\lambda T'(\rho^\star)}
=\frac{-1}{\lambda\,(1-F^V(\rho^\star))}<0.
\]
For a first-order stochastic improvement $F^V_2\le F^V_1$ pointwise, one has
$T_2(\rho)\ge T_1(\rho)$ for all $\rho$; since $\rho\mapsto\lambda T(\rho)$ is
decreasing, the solution of $\lambda T(\rho)=c_w$ moves right, so
$\rho^\star_2\ge\rho^\star_1$.
\end{proof}

\begin{figure}[t]
  \centering
  \includegraphics[width=\linewidth]{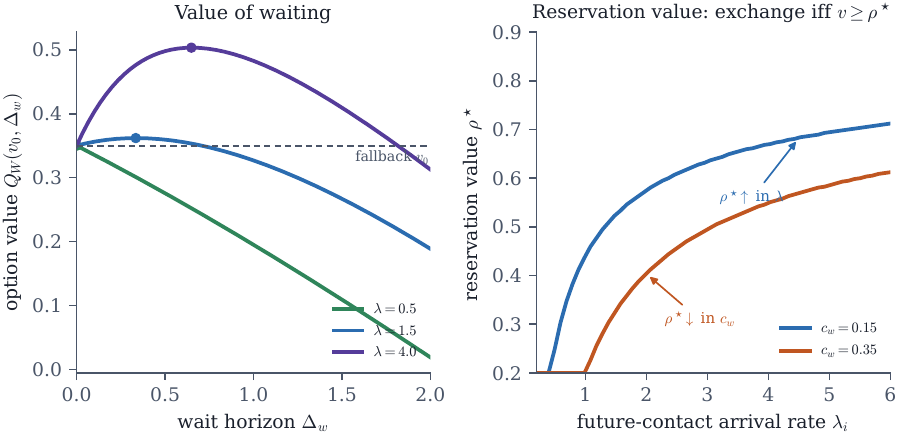}
  \caption{\emph{Left:} the value of waiting $Q_W(v_0,\Delta_w)$
  (\Cref{prop:wait}) as a function of the wait horizon for several arrival rates;
  waiting is attractive only when arrivals are frequent enough to offset delay
  cost. \emph{Right:} the search-theoretic reservation value $\rho^\star$
  (\Cref{thm:reservation}) increases in the arrival rate $\lambda_i$ and decreases
  in the delay price $c_w$ (\Cref{prop:reservationcs}); the learner exchanges now
  iff the realisable value on offer reaches $\rho^\star$.}
  \label{fig:waiting}
\end{figure}

\Cref{thm:reservation} completes the reservation-value picture of
\Cref{prop:snell}: the waiting component of $\rho=W$ is itself a reservation
value, so the full optimal policy exchanges as soon as the certified realisable
value exceeds the larger of the probing continuation and the waiting reservation
$\rho^\star$.

% ==================================================================== section 8
\section{Myopic Optimality of the One-Step Rule}
\label{sec:myopic}

Exact evaluation of the probe continuation \eqref{eq:qprobe} requires the
high-dimensional expectations $\E[V(s^+)]$ and $\E[V(s^{\mathrm{drop}})]$. PROSE
replaces them with the one-step (myopic) surrogates of \Cref{sec:certify,sec:voi}.
We now show when this is not merely tractable but optimal.

Fix the waiting option by folding its value into a within-stage outside option
$\omega=\max\{v_0,Q_R,\rho^\star\}$, and consider the resulting probe-versus-stop
subproblem. Let $B_n$ be the guaranteed value \eqref{eq:vpi} at probing stage $n$
and let $C_n=\sum_{m<n}(c_p+\ell_m)$ be the cumulative acquisition-and-opportunity
cost incurred to reach stage $n$, where $\ell_m\ge0$ is the realised opportunity
loss at stage $m$. The reward from stopping at stage $n$ is $Y_n=B_n-C_n$. The
\emph{one-step-look-ahead} (OSLA) rule stops at the first stage with
$\E[Y_{n+1}-Y_n\mid\mathcal F_n]\le0$, i.e.
\begin{equation}
\E[B_{n+1}-B_n\mid\mathcal F_n]\ \le\ \kappa_n:=c_p+\E[\ell_n\mid\mathcal F_n].
\label{eq:osla}
\end{equation}

\begin{definition}[Monotone regime]
\label{def:monotone}
The probing subproblem is \emph{monotone} if the OSLA stopping set
$A^{\mathrm{OSLA}}_n=\{\E[B_{n+1}-B_n\mid\mathcal F_n]\le\kappa_n\}$ is absorbing:
once entered it is never left along any probing continuation.
\end{definition}

\begin{lemma}[Sufficient conditions for monotonicity]
\label{lem:monotone}
Suppose along every probing continuation (i) the expected one-step improvement in
guaranteed value $\E[B_{n+1}-B_n\mid\mathcal F_n]$ is non-increasing in $n$, and
(ii) the expected per-stage cost $\kappa_n$ is non-decreasing in $n$. Then the
subproblem is monotone. In particular, absent new peer arrivals and with
negligible drift, \Cref{lem:diminish} gives $\VPI_n^{\mathrm{ub}}\le2S_{\max}r_n$
non-increasing, which bounds and forces down the expected improvement, so (i)
holds; and if survival is non-increasing in elapsed contact time, then
$\Omega$ and hence $\kappa_n$ are non-decreasing, so (ii) holds. Both are promoted
by high hazard, which shrinks survivable uncertainty and accelerates opportunity
loss.
\end{lemma}
\begin{proof}
Under (i) the left-hand side of \eqref{eq:osla} is non-increasing and under (ii)
the right-hand side $\kappa_n$ is non-decreasing; hence once
$\E[B_{n+1}-B_n\mid\mathcal F_n]\le\kappa_n$ holds it continues to hold, so
$A^{\mathrm{OSLA}}_n\subseteq A^{\mathrm{OSLA}}_{n+1}$, which is
\Cref{def:monotone}. The two structural claims are immediate: shrinking intervals
make $U_n^{\max}-B_n$ non-increasing, bounding the achievable improvement; and
$\Omega_j=(1-\widehat S_{ij}(\tau_p))[\cdot]_+$ is non-decreasing when survival
$\widehat S_{ij}(\tau_p)$ falls with consumed contact time.
\end{proof}

\begin{theorem}[Myopic optimality and soundness of the confidence-safe rule]
\label{thm:myopic}
Consider the probe-versus-stop subproblem with bounded rewards and finitely many
stages (\Cref{prop:horizon}). Then:
\begin{enumerate}
\item[(a)] \emph{(Optimality of OSLA.)} If the subproblem is monotone
(\Cref{def:monotone}), the one-step-look-ahead rule \eqref{eq:osla} is optimal.
\item[(b)] \emph{(Soundness of the PROSE rule.)} Because
$\VPI_n^{\mathrm{ub}}\ge\E[B_{n+1}-B_n\mid\mathcal F_n]$ and, under the loss
hypothesis of \Cref{thm:vpi}, $\Omega_b\le\E[\ell_n\mid\mathcal F_n]$, the PROSE
do-not-probe set $A^{\mathrm M}_n=\{\VPI_n^{\mathrm{ub}}\le c_p+\Omega_b\}$
satisfies $A^{\mathrm M}_n\subseteq A^{\mathrm{OSLA}}_n$. Hence in the monotone
regime every state in which PROSE halts probing is one in which halting is
optimal: PROSE never stops prematurely.
\item[(c)] \emph{(Bounded over-acquisition.)} PROSE's only possible inefficiency
is continuing to probe in states of $A^{\mathrm{OSLA}}_n\setminus A^{\mathrm
M}_n$; the number of such extra probes is at most the remaining horizon
$\bar n-n$, so the regret of PROSE relative to the optimal policy is at most
$(\bar n-n)\,(c_p+\max_m\ell_m)$.
\end{enumerate}
\end{theorem}
\begin{proof}
(a) For a monotone finite-horizon stopping problem with bounded rewards, the OSLA
rule is optimal \citep{chow1971,ferguson2006}: writing $V_n=\E[Y_{\tau^\star}\mid\mathcal F_n]$ for the optimal value from stage $n$, backward induction shows
that on the absorbing set $A^{\mathrm{OSLA}}$ continuation cannot help, because
each subsequent expected increment is $\le0$ and increments are non-improving, so
stopping on first entry attains $V_n$.

(b) On the coverage event of \Cref{lem:stale} the refreshed lower bounds cannot
exceed the current upper bounds, so $B_{n+1}\le U_n^{\max}$ and hence
$\VPI_n^{\mathrm{ub}}=U_n^{\max}-B_n$ upper-bounds the improvement in guaranteed
value from perfect information, which itself upper-bounds the expected one-step
improvement $\E[B_{n+1}-B_n\mid\mathcal F_n]$ from a single probe (off that event,
of probability at most $\delta$, an additive slack $\delta(\bar g-\underline g)$
applies); and by hypothesis $\Omega_b\le\E[\ell_n\mid\mathcal F_n]$. Thus if
$\VPI_n^{\mathrm{ub}}\le c_p+\Omega_b$ then
$\E[B_{n+1}-B_n\mid\mathcal F_n]\le\VPI_n^{\mathrm{ub}}\le c_p+\Omega_b\le\kappa_n$,
i.e.\ $A^{\mathrm M}_n\subseteq A^{\mathrm{OSLA}}_n$. In the monotone regime
$A^{\mathrm{OSLA}}$ is the optimal stopping region by (a), so halting anywhere in
$A^{\mathrm M}_n$ is optimal.

(c) Every stage in which PROSE probes while OSLA would stop lies in
$A^{\mathrm{OSLA}}_n\setminus A^{\mathrm M}_n$. Each such probe costs at most
$c_p+\max_m\ell_m$ and, by \Cref{prop:horizon}, at most $\bar n-n$ stages remain;
summing gives the stated regret bound.
\end{proof}

\begin{remark}
\Cref{thm:myopic} formalises a design principle: volatility is the friend of
myopia. Precisely the conditions that make careful multi-step planning attractive
in stable networks---persistent contacts, slowly drifting peers---are those in
which the monotone property may fail; conversely, high hazard both accelerates
opportunity loss and shrinks survivable uncertainty (\Cref{lem:monotone}), driving
the problem into the monotone regime where the cheap one-step rule is optimal. The
containment $A^{\mathrm M}_n\subseteq A^{\mathrm{OSLA}}_n$ is a property of the
bounds alone and so holds regardless of monotonicity: PROSE halts probing only
where the one-step rule would. Whether such a halt is \emph{optimal} (no premature
stopping) requires monotonicity, under which $A^{\mathrm{OSLA}}$ coincides with the
optimal stopping region.
\end{remark}

% ==================================================================== section 9
\section{The Induced Policy}
\label{sec:policy}

\Cref{alg:prose} assembles the theory into a fully local policy. It uses only
locally observable quantities: contact durations and disconnections feed the
survival estimate \eqref{eq:hazard}; recent arrivals and realised peer values feed
$(\widehat\lambda_i,\widehat F_i^V)$ and hence $Q_W$ and $\rho^\star$; completed
probes feed the drift-aware bounds \eqref{eq:driftlower}--\eqref{eq:driftci} and
the realisable-value bounds \eqref{eq:valuebounds}. No global topology,
server-side validation set, or future-contact oracle is assumed. The decision
logic is $O(|\Nset_i(t)|)$ per stage plus $O(W_c\log W_c)$ for maintaining the
empirical mark buffer of size $W_c$, negligible beside a forward pass over anchors
and, especially, a model transfer. The controlled systems cost is therefore not
arithmetic but elapsed contact time: every probe consumes $\tau_p$ and can change
which actions remain feasible.

\begin{algorithm}[t]
\caption{PROSE decision episode at receiver $i$ (confidence-safe one-step policy)}
\label{alg:prose}
\begin{algorithmic}[1]
\Require reachable peers $\Nset_i(t)$, anchors $\Aset_i$, deadline $H$, cap $n_{\max}$, level $\delta$
\State Update $\widehat S_{ij}$ from contact history via \eqref{eq:hazard}; update $(\widehat\lambda_i,\widehat F_i^V)$; compute $Q_W$ via \eqref{eq:wait} and $\rho^\star$ via \eqref{eq:reseqn}
\While{elapsed decision time $<H$}
  \State Refresh $\Nset_i(t)$; \textbf{if} empty \textbf{then return} $\argmax\{Q_W,v_0\}$
  \ForAll{$j\in\Nset_i(t)$}
     \State If $n_{ij}=0$ set $(L_{ij},U_{ij})=(0,1)$; else compute \eqref{eq:evidence}--\eqref{eq:driftci}
     \State Compute $(\underline V_{ij},\overline V_{ij})$ via \eqref{eq:valuebounds} and $\Omega_j$ via \eqref{eq:opp}
  \EndFor
  \State $b\gets\argmax_j\underline V_{ij}$ \Comment{maximin incumbent, \eqref{eq:incumbent}}
  \If{\eqref{eq:certify} holds \textbf{and} $\underline V_{ib}\ge\max\{Q_W,\rho^\star\}$}
     \State \Return \textsc{Exchange}$(b)$
  \EndIf
  \State Compute $(B_n,U_n^{\max})$, $\VPI_n^{\mathrm{ub}}$ via \eqref{eq:vpi}, and $I_{ij}$ via \eqref{eq:index}
  \State $\mathcal P\gets\{j:I_{ij}>0,\ n_{ij}<n_{\max}\}$
  \If{\eqref{eq:noprobe} is false \textbf{and} $\mathcal P\neq\emptyset$}
     \State Probe $j^\star\gets\argmax_{j\in\mathcal P}I_{ij}$; on completion increment $n_{ij^\star}$, else drop $j^\star$
  \Else
     \State $Q_E\gets\underline V_{ib}$;\ \ $Q_R\gets|\Nset_i|^{-1}\sum_j\underline V_{ij}$
     \State \Return best feasible action in $\{Q_E,Q_W,Q_R,v_0\}$
  \EndIf
\EndWhile
\State \Return best feasible current exchange/random action, or $v_0$ if $\Nset_i(t)=\emptyset$
\end{algorithmic}
\end{algorithm}

% ==================================================================== section 10
\section{Special Cases and Limits}
\label{sec:limits}

\begin{corollary}[Static-contact limit]
\label{cor:static}
If $S_{ij}(\Delta)=1$ for all relevant $\Delta$ and no new peers arrive, then
$\Omega_j=0$, realisable values reduce to statistical values minus exchange cost
($v_{ij}=\mu_{ij}-c_e$), and $Q_W$ has no positive option value. The problem
reduces to ordinary sequential evidence acquisition with a stop-versus-probe
choice, and \Cref{thm:vpi} becomes the classical value-of-information rule
$\VPI_n^{\mathrm{ub}}\le c_p$.
\end{corollary}

\begin{corollary}[Drift-free identification limit]
\label{cor:nodrift}
If additionally $\nu_{ij}=0$, the intervals \eqref{eq:driftlower}--\eqref{eq:driftci}
reduce to Hoeffding bands of half-width $r_n$, and certification
\eqref{eq:certify} is the stopping rule of fixed-confidence best-arm
identification \citep{evendar2006,jamieson2014}; \Cref{thm:samples} recovers the
familiar $O(\Delta_v^{-2}\log(K/\delta))$ sample complexity.
\end{corollary}

\begin{lemma}[Model-bearing communication bound]
\label{lem:comm}
Let a full model have $P$ parameters and one probe return at most $d_a$ scalars.
An episode with $n$ probes and at most one exchange transmits at most $nd_a+P$
scalars ($nd_a$ if no exchange occurs), whereas indiscriminately exchanging with
$m$ candidates transmits $mP$ model scalars. When $d_a\ll P$, sequential
evaluation remains lightweight even with several probes; the communication ratio
of one episode to one model transfer is $\Gamma(n)=1+nd_a/P$.
\end{lemma}
\begin{proof}
Each probe adds at most $d_a$ scalars and at most one model of $P$ scalars is
sent, giving the payload bound and $\Gamma(n)$ on division by $P$; the alternative
sends one $P$-scalar model per candidate. The inequality $nd_a+P\le mP$ for
$d_a\ll P$ and moderate $n<(m-1)P/d_a$ is immediate.
\end{proof}

% ==================================================================== section 11
\section{Discussion}
\label{sec:discussion}

The theory separates peer selection from the underlying decentralised optimiser,
which clarifies the decision problem and sets the scope of the guarantees. They
concern peer certification and stopping, not end-to-end non-convex FL convergence;
establishing convergence with policy-dependent, temporally correlated exchanges is
a natural extension, for which the reservation-value structure of
\Cref{prop:snell} and the myopic-optimality regime of \Cref{thm:myopic} provide a
principled selection layer to build on. The exponential hazard model
\eqref{eq:hazard} is deliberately simple; the policy consumes only the survival
values $S_{ij}(\tau_p),S_{ij}(\tau_e)$, so non-parametric or context-conditioned
estimators may be substituted, and the comparative statics of \Cref{sec:voi}
depend only on survival being decreasing in consumed contact time. Anchor quality
governs how well pre-exchange compatibility predicts post-merge benefit; the
drift envelope $\nu_{ij}$ is treated as valid in the coverage statements, and
estimating it from the same probe stream is an implementation heuristic unless an
additional coverage argument is supplied. The Poisson waiting model is an
approximation whose role is to expose the option value of future contacts in a
transparent, closed form; empirical mark distributions can be substituted directly
in \eqref{eq:wait} and \eqref{eq:reseqn}.

A broader implication is that uncertainty in mobile decentralised learning has two
clocks. Statistical uncertainty falls as probes accumulate, while opportunity
uncertainty rises as contact lifetime is consumed. A good policy therefore should
not maximise confidence or link stability in isolation; it should stop when the
marginal value of information falls below the combined cost of probing, possible
disconnection, and delayed access to future peers. The results above make each of
these three costs explicit and show, in the monotone regime, that acting on them
one step at a time is optimal.

% ==================================================================== section 12
\section{Conclusion}
\label{sec:conclusion}

We developed a theory of optimal stopping for opportunistic peer selection in
decentralised federated learning over transient contact graphs, treating peer
evidence as perishable. We formalised the within-contact decision as a
finite-horizon Markov optimal-stopping problem, proved existence of an optimal
policy and its reservation-value (Snell-envelope) structure, and built around it a
confidence-safe certification theory, a mobility-aware value-of-information
stopping rule with hazard comparative statics, a closed-form and a
search-theoretic characterisation of the waiting option, and a myopic-optimality
theorem showing that in volatile regimes the lightweight one-step rule is a sound
and, under monotonicity, optimal surrogate for the full dynamic program. The
resulting policy, PROSE, is fully local and inexpensive, and it degenerates
gracefully to classical sequential evidence acquisition and best-arm
identification in the static, drift-free limit. The contribution is theoretical:
a precise account of when, whether, and with whom a model exchange should occur
when the evidence for that decision decays even as it is gathered.

% ==================================================================== references

\end{document}